\documentclass{article}

     \PassOptionsToPackage{numbers, sort&compress}{natbib}

     \usepackage[final]{neurips_2020}

\usepackage[utf8]{inputenc} 
\usepackage[T1]{fontenc}    
\usepackage{hyperref}       
\usepackage{url}            
\usepackage{booktabs}       
\usepackage{amsfonts}       
\usepackage{nicefrac}       
\usepackage{microtype}      

\usepackage{dot2texi}
\usepackage{tikz}
\usepackage{caption}
\usepackage{subcaption}
\usetikzlibrary{shapes,arrows}
\makeatletter
\@ifundefined{verbatim@out}{\newwrite\verbatim@out}{}
\makeatother
\usepackage[utf8]{inputenc}
\usepackage{amsmath}
\usepackage{amssymb}
\usepackage{amsthm}
\usepackage{xcolor}
\usepackage{physics}
\usepackage{amsmath,amsfonts,amssymb,amsthm}
\usepackage{dsfont}
\usepackage{todonotes}

\newcommand{\ints}{\int_\mathcal{S}}

\renewcommand{\d}{\mathrm{d}}
\DeclareMathOperator*{\argmax}{arg\,max}
\DeclareMathOperator*{\argmin}{arg\,min}
\usepackage{thmtools}
\usepackage{thm-restate}

\usepackage{algorithm}
\usepackage{algpseudocode}
\usepackage{subcaption}
\usepackage{booktabs}
\usepackage{lipsum}
\usepackage{hhline}
\usepackage{mathtools}
\usepackage{wrapfig}
\usepackage{array}
\newcolumntype{P}[1]{>{\centering\arraybackslash}p{#1}}
\newcolumntype{M}[1]{>{\centering\arraybackslash}m{#1}}
\algnewcommand{\LineComment}[1]{\State \(\triangleright\) #1}

\usepackage{bm}
\usepackage{amssymb}
\newcommand{\vtheta}{{\bm{\theta}}}
\newcommand{\vphi}{{\bm{\phi}}}
\newcommand{\vomega}{{\bm{\omega}}}

\usepackage{xspace}
\DeclareRobustCommand{\eg}{e.g.,\@\xspace}
\DeclareRobustCommand{\ie}{i.e.,\@\xspace}
\DeclareRobustCommand{\wrt}{w.r.t.\@\xspace}

\newcommand{\E}{\mathop{\mathbb{E}}}

\usepackage{etoolbox}

\title{How to Learn a Useful Critic? Model-based Action-Gradient-Estimator Policy Optimization}

\author{%
  Pierluca D'Oro\thanks{Work done while at NNAISENSE.}\\
  MILA, Université de Montréal\\
  \texttt{pierluca.doro@mila.quebec} \\
   \And
   Wojciech Jaśkowski \\
   NNAISENSE \\
   \texttt{wojciech@nnaisense.com} \\
}

\begin{document}

\maketitle

\begin{abstract}
Deterministic-policy actor-critic algorithms for continuous control improve the actor by plugging its actions into the critic and ascending the action-value gradient, which is obtained by chaining the actor's Jacobian matrix with the gradient of the critic with respect to input actions.
However, instead of gradients, the critic is, typically, only trained to accurately predict expected returns, which, on their own, are useless for policy optimization.
In this paper, we propose MAGE, a model-based actor-critic algorithm, grounded in the theory of policy gradients, which explicitly learns the action-value gradient.
MAGE backpropagates through the learned dynamics to compute gradient targets in temporal difference learning, leading to a critic tailored for policy improvement.
On a set of MuJoCo continuous-control tasks, we demonstrate the efficiency of the algorithm in comparison to model-free and model-based state-of-the-art baselines.
\end{abstract}

\section{Introduction}
Reinforcement learning (RL)~\cite{Puterman1994MarkovDP,sutton2018reinforcement} studies sequential decision making problems, in which an agent aims at maximizing the cumulative reward it collects in an environment.
One of the most popular classes of algorithms for RL are \emph{policy gradient methods}~\cite{sutton2000policy,deisenroth2013survey}, which involve differentiable control policies improved by gradient ascent. They feature suitability to environments with continuous state and action spaces, and compatibility with state-of-the-art deep learning~\cite{schmidhuber2015deep} methods.
Policy gradient algorithms often employ an actor-critic~\cite{konda2000actor} scheme: an \emph{actor}, which determines the control policy, is evaluated using a \emph{critic}.
Thus, the degree of actor's improvement is limited by the information provided by the critic, naturally raising the question of how the critic should be trained.

Typically, algorithms that use powerful function approximators~\cite{lillicrap2015continuous,fujimoto2018addressing} learn the critic by temporal difference~\cite{sutton1988learning}, optimizing for an accurate prediction of the expected return of the actor.
For deterministic-policy continuous-control~\cite{silver2014deterministic,lillicrap2015continuous}, however, the value provided by the critic is neither used for improving the policy nor for acting in the environment~\cite{sutton2000policy}.
Instead, only the \emph{action-gradient} of the value function, \ie the gradient of the critic \wrt the action performed by the actor, is employed during policy optimization.
Specifically, the policy gradient is obtained through the computation of the \emph{action-value gradient}, by chaining the actor's Jacobian with the action-gradient of the critic.

Learning the critic \emph{by value} rather than \emph{by action-gradient} of the value relies on hazy smoothness assumptions on the real value function~\cite{silver2014deterministic}.
This means that, in conventional temporal difference learning, the critic learns action-value gradients \emph{implicitly}, which could harm the performance of a deterministic policy gradient algorithm.

In this paper, we propose Model-based Action-Gradient-Estimator Policy Optimization (MAGE), a continuos-control deterministic-policy actor-critic algorithm that \emph{explicitly} trains the critic to provide accurate action-gradients for the use in the policy improvement step.
Motivated by both the theory on \emph{Deterministic Policy Gradients}~\cite{silver2014deterministic} and practical considerations, MAGE uses temporal difference methods to minimize the error on the action-value gradient.
For this, the algorithm leverages a trained dynamics model as a proxy for a differentiable environment and techniques reminiscent of double backpropagation~\cite{drucker1992improving}.
On a challenging continuous control benchmark~\cite{brockman2016gym,todorov2012mujoco}, we show that MAGE is significantly more sample-efficient than state-of-the-art model-free and model-based baselines.

The rest of the paper is organized as follows.
In Section~\ref{sec:background}, we provide the notation and background on deterministic policy gradients.
Our algorithm, together with its theoretical motivation, is introduced in Section~\ref{sec:algorithm}, followed by empirical results in Section~\ref{sec:experiments}.
In Section~\ref{sec:relatedWork}, we present some of the related work and its relationship with our approach.

\section{Background}\label{sec:background}
\subsection{Preliminaries}
Consider a discrete-time Markov Decision Process~\cite{Puterman1994MarkovDP} (MDP), defined as ${\mathcal{M} = \left( \mathcal{S}, \mathcal{A}, p, r, \gamma, \mu \right)}$, where $\mathcal{S}$ is the space of possible states, $\mathcal{A}$ is the space of possible actions, $p: \mathcal{S} \times \mathcal{A} \rightarrow \Delta \left( \mathcal{S} \right)$ is the transition model, $r: \mathcal{S} \times \mathcal{A} \rightarrow \mathbb{R}$ is the known and differentiable reward function, $\gamma$ is the discount factor, $\mu \in \Delta \left( \mathcal{S} \right)$ is the initial state distribution.
The behavior of the agent is described by a deterministic policy $\pi_{\vtheta}: \mathcal{S} \rightarrow \mathcal{A}$, belonging to a parametric space of policies $\Pi = \{ \pi_\vtheta :\vtheta \in \Theta \subseteq \mathbb{R}^n \}$, for which we will occasionally omit the parameter subscript.
Let $d_\mu^{\pi}$ be the $\gamma$-discounted state distribution induced by policy $\pi_{\vtheta}$, defined as $d_\mu^\pi(s) = (1 - \gamma) \sum_{t=0}^\infty \gamma^t \Pr (s_t = s | \pi, \mu)$.
The total reward collected by an agent is quantified with action-value function $Q^\pi(s,a) = \E \left[ \sum_{t=0}^{\infty} \gamma^t r(s_t, a_t) | s_0 = s, a_0 = a\right]$ and performance function $J(\vtheta)= \E_{s \sim \mu} \left[ Q^\pi(s,\pi_{\vtheta}(s)) \right]$.

Practical algorithms can employ an approximate action-value function $\widehat{Q}$ and an approximate dynamics model~$\widehat{p}$, which, most commonly, are parametric function approximators specified by the spaces $\mathcal{Q} = \{ Q_\vphi :\vphi \in \Phi \subseteq \mathbb{R}^h \}$ and $\mathcal{P} = \{ p_\vomega :\vomega \in \Omega \subseteq \mathbb{R}^k \}$.

\subsection{Deterministic Policy Gradients and TD-learning}
Policy gradient methods improve the policy $\pi_\vtheta$ by ascending the direction of the gradient of its performance function $J(\vtheta)$.
The \emph{Deterministic Policy Gradient Theorem}~\cite{silver2014deterministic} provides a practical way to calculate this gradient. It shows that, under some mild regularity conditions on the MDP, the gradient of the performance of a deterministic policy $\pi_{\vtheta}$ is given by:
\begin{equation}
    \nabla_{\vtheta} J(\vtheta) = \frac{1}{1 - \gamma} \ints d_\mu^\pi(s)  \nabla_a Q^\pi(s,a) \big|_{a = \pi_{\vtheta}(s)} \nabla_{\vtheta} \pi_{\vtheta} (s) \d s.
\end{equation}
This result can be interpreted through the lens of the chain rule applied to the \emph{action-value gradient} $\nabla_{\vtheta} Q^\pi$: the policy gradient does not directly depend on the gradient of $d_\mu^\pi$, and can be obtained by just chaining the actor's Jacobian $\nabla_{\vtheta} \pi_{\vtheta}$ with the \emph{action-gradient} of the value function $\nabla_a Q^\pi$.

The theorem motivates a family of policy gradient actor-critic algorithms, such as DDPG \cite{lillicrap2015continuous} and TD3 \cite{fujimoto2018addressing}.
Similarly to the classical policy iteration~\cite{sutton2018reinforcement}, the evaluation of a policy $\pi \in \Pi$ (called \emph{actor} in this context) is interleaved with its improvement w.r.t the approximate action-value function $\widehat{Q} \in \mathcal{Q}$ (called \emph{critic}).
Specifically, the typical desideratum consists in finding a critic $\widehat{Q}$ which minimizes the \emph{policy evaluation error}:
\begin{equation}
\widehat{Q} \in \argmin_{\widetilde{Q} \in \mathcal{Q}} \E_{s \sim d_\mu^\pi} \left| \delta^{\pi, \widetilde{Q}}(s,\pi(s)) \right|,
\end{equation}
where $\delta^{\pi,\widetilde{Q}}(s,a) = Q^\pi(s,a) - \widetilde{Q}(s,a)$ is a deviation w.r.t the true state-action value.
Given the lack of knowledge about the transition model, $Q^\pi$ needs to be approximated.
A common approximation technique consists in employing the \emph{temporal difference} (TD) \emph{error}~\cite{sutton1988learning}, defined as ${\widehat{\delta}^{\pi,\widetilde{Q}}(s,a,s') = r(s,a) + \gamma \widetilde{Q}(s',\pi(s')) - \widetilde{Q}(s, a)}$, giving rise to a bootstrapped optimization criterion for $\widehat{Q}$:
\begin{equation}\label{eq:tderror}
    \widehat{Q} \in \argmin_{\widetilde{Q} \in \mathcal{Q}} \E_{\substack{\text{$s \sim d_\mu^\pi$}\\\text{$s' \sim p(\cdot|s,\pi(s))$}}} \left| \widehat{\delta}^{\pi,\widetilde{Q}}(s, \pi(s), s') \right|.
\end{equation}
Minimizing the TD-error, albeit under rather strong assumptions, enjoys convergence guarantees~\cite{sutton2018reinforcement,tsitsiklis1997analysis}.
Once a critic is learned, the actor $\pi_{\vtheta}$ can be improved by maximizing the action-value function for actions produced by the policy:
\begin{equation}\label{eq:argmax_dpg}
    \pi_{\vtheta} \in \argmax_{\widetilde{\pi}_{\vtheta} \in \Pi_{\Theta}} \E_{s \sim d_\mu^\pi} \left[ \widehat{Q}(s, \widetilde{\pi}_{\vtheta}(s)) \right].
\end{equation}
The above can be seen as a generalization of the policy improvement step in classical policy iteration, which relies on maximization over a discrete action space that cannot be easily carried out in continuous spaces.
In practice, to reduce the computational burden, the problems in Equation~\ref{eq:tderror} and Equation~\ref{eq:argmax_dpg} are solved only partially (\eg by using a single optimization step) at each iteration, similarly to generalized policy iteration~\cite{sutton2018reinforcement}.

\section{Learning Action-Value Gradients}\label{sec:algorithm}
In this section, we discuss theoretically how to learn a useful critic in the context of deterministic policy gradients.
Then, we make the theoretical insights concrete and, guided by practical considerations, present \emph{Model-based Action-Gradient-Estimator Policy Optimization} (MAGE), a novel policy optimization algorithm.

\subsection{How to Learn a Useful Critic?}
An actor can only be as good as allowed by its critic.
Thus, obtaining an \emph{effective} critic is one of the most crucial passages for any actor-critic algorithm.
In the previous section, we outlined the most common method to train the critic, consisting in the minimization of the temporal difference error.
However, when the learned action-value function will not be perfect, as common in policy optimization with function approximation, minimizing the TD-error does not guarantee that the critic will be effective for the goal of solving the control problem.
Instead, the following result provides foundations for a more grounded objective function for critic learning.
\begin{restatable}{proposition}{gradientLossBound}
\label{th:gradient_loss_bound}
Let $\Pi$ be a parametric space of $L_\pi$-Lipschitz continuous differentiable deterministic policies, $\mathcal{Q}$ a space of approximate value functions and $\| \cdot \|$ any $p$-norm. Given $\pi \in \Pi$ and $\widehat{Q} \in \mathcal{Q}$, the norm of the difference between the true policy gradient $\nabla_{\vtheta} J(\vtheta)$ and its approximation $\widehat{\nabla}_{\vtheta} J (\vtheta)$, which uses $\widehat{Q}$, can be upper bounded as:
\begin{equation*}\label{prop:normgrad}
    \| \nabla_{\vtheta} J (\vtheta) - \widehat{\nabla}_{\vtheta} J (\vtheta) \|  \leq  \frac{L_\pi}{1 - \gamma}  \E_{s \sim d_\mu^\pi} \left\| \nabla_a \delta^{\pi,\widehat{Q}}(s,a) \Big|_{a=\pi(s)}  \right\|.
\end{equation*}
\end{restatable}
The proposition (see Appendix~\ref{apx:proofs} for the proof) is a direct consequence of the \emph{Deterministic Policy Gradient Theorem} and is thus valid when deterministic policies are employed. The Lipschitz assumption for $\pi$ is easily satisfied for many policy classes of practical use, \eg neural networks~\cite{fazlyab2019efficient}.

Proposition~\ref{th:gradient_loss_bound} suggests that it is the norm of the action-gradient of the policy evaluation error instead of its value that should be minimized to reduce the bias introduced by the use of the approximate value function~$\widehat{Q}$.
To minimize the bound, a proxy for the unknown $Q^\pi$ is needed. To this aim, it is possible to follow the approach of traditional TD-learning, substituting the evaluation error $\delta^{\pi, \widehat{Q}}$ with the TD-error $\widehat{\delta}^{\pi, \widehat{Q}}$. This leads to the following optimization problem:
\begin{equation}
    \widehat{Q} \in \argmin_{\widetilde{Q} \in \mathcal{Q}} \E_{\substack{\text{$s \sim d_\mu^\pi$}\\\text{$s' \sim p(\cdot|s,\pi(s))$}}} \left\| \nabla_a \widehat{\delta}^{\pi,\widetilde{Q}}(s, \pi(s), s') \right\|.
\end{equation}
Notice that computing the gradient \wrt the action of the TD-error $\widehat{\delta}^{\pi,\widehat{Q}}$ requires taking into account the effect of action $a$ on the transition to the subsequent state in the environment $s'$, i.e., backpropagating through the environment dynamics $p$.
Since $p$ is not available in typical RL settings, especially in a differentiable form, it needs to be substituted with an approximate model $\widehat{p}$, as commonly done in model-based RL~\cite{deisenroth2013survey,janner2019trust,Chua2018DeepRL}.
An environment model gives rise to imaginary transitions $(s,\pi(s),\widehat{s})$, where $\widehat{s} \sim \widehat{p}(\cdot|s,\pi(s))$.
Given differentiable model, policy, and action-value function, the action-gradient can be effectively computed by leveraging standard automatic differentiation tools~\cite{baydin2018autodiff}.
The corresponding computational graph is depicted in Figure~\ref{fig:computational_graph2}.
This leads to a viable way to obtain $\widehat{Q}$:
\begin{equation}\label{eq:tdgerror}
    \widehat{Q} \in \argmin_{\widetilde{Q} \in \mathcal{Q}} \E_{\substack{\text{$s \sim d_\mu^\pi$}\\\text{$\widehat{s} \sim \widehat{p}(\cdot|s,\pi(s))$}}} \left\| \nabla_a \widehat{\delta}^{\pi,\widetilde{Q}}(s, \pi(s), \widehat{s}) \right\|.
\end{equation}
Even in the general case of a stochastic model, differentiating through the resulting computations is still possible for many commonly used model classes via the reparametrization trick~\cite{heess2015learning}.
Using an approximate model $\widehat{p}$ implies a tradeoff, since additional bias is injected into the estimation of the critic.
Nonetheless, the use of $\widehat{p}$ is the most direct way to solve the optimization problem in Equation~\ref{eq:tdgerror} and to obtain a $\widehat{Q}$ that provides a more accurate policy gradient \wrt the typical critic.

\begin{figure}
    \centering
    \includegraphics[width=\textwidth]{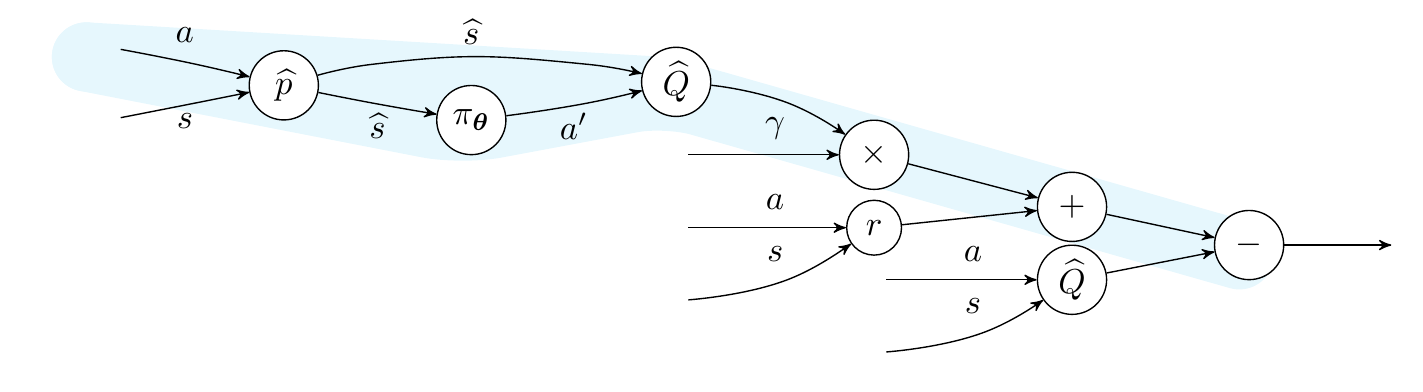}
\caption{Graph describing the computation of $\widehat{\delta}^{\pi_{\vtheta},\widehat{Q}}$, when using policy $\pi_{\vtheta}$, model $\widehat{p}$, action-value function $\widehat{Q}$. Nodes and edges represent functions and variables, respectively. To compute $\nabla_a \widehat{\delta}^{\pi_{\vtheta},\widehat{Q}}$, all the paths from the output back to $a$ must be considered, including the one highlighted in cyan, which involves the environment dynamics. Therefore, an approximate differentiable model $\widehat{p}$ needs to be learned in order to make all the required paths accessible.}
    \label{fig:computational_graph2}
\end{figure}

\subsection{Model-based Action-Gradient-Estimator Policy Optimization}
The outlined procedure for learning the value function requires an approximate model $p_{\omega}$, thus naturally suggesting its integration into a \mbox{model-based} policy optimization framework.
A model-based actor-critic method involves three steps during each iteration: learning the model $p_{\vomega}$, updating the action-value function $Q_{\vphi}$ and improving the policy $\pi_{\vtheta}$.
In the following, we consider neural networks as function approximators to represent the three modules, although any class of differentiable models could be leveraged.
Our approach is inspired by Dyna~\cite{sutton1991dyna}, and employs an approximate dynamics model for generating 1-step imaginary on-policy transitions starting from observed states stored in a replay buffer.
Those transitions are then employed to learn $Q_{\vphi}$, and, in turn, leveraged for computing an improvement direction for the parameters of the policy $\pi_{\vtheta}$.

In preliminary experiments, we found that directly solving the minimization problem in Equation~\ref{eq:tdgerror} is hard in practice.
During the optimization, the parameters are prone to be trapped in local-minima, which leads to degenerate solutions.
A demonstration of this phenomenon is detailed in Appendix~\ref{apx:degenerate}.
The root cause of this effect is unknown and suggests the existence of a tradeoff between the easier minimization of the TD-error and the more theoretically grounded minimization of its action-gradient.

We propose as a remedy the introduction of a constraint into the optimization problem.
We argue that, among the possible solutions, a natural one is constraining the optimization landscape by bounding the traditional TD-error (see Equation~\ref{eq:tderror}), and thus solving the following optimization problem:
\begin{equation}
\begin{gathered}\label{eq:constrained_problem}
        \min_{\widetilde{\vphi} \in \Phi} \E_{\substack{\text{$s \sim d_\mu^\pi$}\\\text{$\widehat{s} \sim p_{\vomega}(\cdot|s,\pi(s))$}}}\left\| \nabla_a \widehat{\delta}^{\pi, Q_{\widetilde{\vphi}}}(s,a,\widehat{s})\Big|_{a=\pi(s)}  \right\| \\
        \mathrm{s.t. } \E_{\substack{\text{$s \sim d_\mu^\pi$}\\\text{$\widehat{s} \sim p_{\vomega}(\cdot|s,\pi(s))$}}} \left| \widehat{\delta}^{\pi, Q_{\widetilde{\vphi}}}(s,\pi(s),\widehat{s}) \right| \leq \lambda.
\end{gathered}
\end{equation}
As the above expressions already require non-trivial gradient computations, we avoid the use of complex and expensive methods for nonlinear programming.
Instead, we resort to \emph{penalty function methods}~\cite{smith1995penalty} by regularizing the original objective by using the TD-error.
A similar approach has been used in the past in, e.g., Proximal Policy Optimization (PPO,  \cite{schulman2017proximal}) to approximately solve different constrained optimization problems.

Eventually, the parameters of $Q_{\vphi}$ are learned by descending the gradient
\begin{equation}\label{eq:regularization}
\nabla_{\vphi} \mathcal{L}(s,a,\widehat{s};\vphi, \vtheta, \vomega) = \nabla_{\vphi} \left( \left\| \nabla_a \widehat{\delta}^{\pi_{\vtheta},Q_{\vphi}}(s, a, \widehat{s}) \big|_{a=\pi_{\vtheta}(s)} \right\| + \lambda \left| \widehat{\delta}^{\pi_{\vtheta},Q_{\vphi}}(s, a, \widehat{s}) \right|  \right)
\end{equation}
on an imaginary transition $(s,a,\widehat{s})$.
This expression requires computing second-order gradients, which would be computationally expensive if computed \wrt to the high-dimensional space of parameters $\Phi$ of the $Q$-function. Here, however, the optimization is affordable since the gradients are computed \wrt, typically low dimensional, actions.
Notice also that the computational overhead of the second term in Equation~\ref{eq:regularization} is minimal, since evaluating the TD-error $\widehat{\delta}^{\pi, Q_{\vphi}}(s,a,\widehat{s})$ is anyway, when using automatic differentiation, required to compute its gradient.

\begin{algorithm}[t]
\small
\caption{Model-based Action-Gradient-Estimator Policy Optimization (MAGE)}
\label{alg:MAGE}
\hspace*{\algorithmicindent} \textbf{Input:} Initial buffer $\mathcal{B}$, set of parameter vectors $\left\{ \vomega, \vphi, \vtheta \right\}$
\begin{algorithmic}
\For{each iteration}
\State Collect transition $(s,a,s')$ acting according to exploratory version of $\pi_{\vtheta}$
\State $\mathcal{B} \gets \mathcal{B} \cup \left\{ (s,a,s') \right\}$
\For{each model learning step}
\State $\vomega \gets \vomega - \alpha_{p} \nabla_{\vomega} \ell(s,a,s'; \vomega), \qquad (s,a,s') \sim \mathcal{B}$
\EndFor
\For{each policy optimization step}
\State Extract state $s$ after sampling $(s,\cdot,\cdot) \sim \mathcal{B}$
\State $\Bar{\vphi} \gets \vphi$
\State $\widehat{\delta}(s,a,\widehat{s};\vphi) \gets r(s,a) + \gamma Q_{\Bar{\vphi}}(\widehat{s}, \pi_{\vtheta}(\widehat{s})) - Q_{\vphi}(s, a)$, \qquad $a=\pi_{\vtheta}(s), \, \widehat{s} \sim p_{\vomega}(\cdot|s,a)$
\State $\vphi \gets \vphi - \alpha_{Q} \nabla_{\vphi} \left( \left\| \nabla_a \widehat{\delta}(s, a, \widehat{s};\vphi) \big|_{a=\pi_{\vtheta}(s)} \right\| + \lambda \left| \widehat{\delta}(s, a, \widehat{s};\vphi) \right|  \right)$
\State $\vtheta \gets \vtheta + \alpha_{\pi} \nabla_\vtheta Q_{\vphi}(s,\pi_{\vtheta}(s))$
\EndFor
\EndFor
\end{algorithmic}
\end{algorithm}
We plug our critic training method into a model-based Dyna-like algorithm, giving rise to \emph{Model-based Action-Gradient-Estimator Policy Optimization} (MAGE), which is presented\footnote{For simplicity of presentation, an abstract version of MAGE is considered in Algorithm~\ref{alg:MAGE}. Any actor-critic algorithm  with deterministic actor can be then used to instantiate MAGE into a practical incarnation.}  in Algorithm~\ref{alg:MAGE}.
At each iteration, the dynamics model $p_{\vomega}$ is trained to maximize the likelihood of the transitions stored in the experience replay buffer $\mathcal{B}$, or, equivalently, to minimize an appropriate loss function $\ell$:
\begin{equation}
    \vomega \in \argmin_{\widetilde{\vomega} \in \Omega} \E_{(s,a,s') \sim \mathcal{B}} \left[ \ell(s,a,s'; \widetilde{\vomega}) \right].
\end{equation}
Then, for one or more steps, the TD-error for the current policy and action-value function is computed, and used together with its action-gradient to update $Q_{\vphi}$, which in turn is leveraged to improve $\pi_{\vtheta}$.

\section{Experiments}\label{sec:experiments}
\subsection{Sample-Efficient Continuous Control with MAGE}
\paragraph{Algorithm settings}
The general structure of MAGE is compatible with many actor-critic algorithms with deterministic policies.
In this experiment, we employ TD3~\cite{fujimoto2018addressing}, a popular, state-of-the-art extension to DDPG~\cite{lillicrap2015continuous}, as a base policy optimization method.
This amounts to the addition of target policy smoothing, delayed policy updates, clipped double Q-learning and target functions.
We call this version of our algorithm MAGE-TD3\footnote{The PyTorch~\cite{paszke2019pytorch} implementation, based on~\cite{shyam2019model}, is available at \url{https://github.com/nnaisense/MAGE}.}.
After each step of environment interaction, we add the collected transition in the replay buffer $\mathcal{B}$, train the approximate model $p_{\vomega}$, and update critic and actor $10$ times.
We employ a single value of $\lambda=0.2$ for all the environments, since we found MAGE to be reasonably robust to the choice of this hyperparameter (see Appendix~\ref{apx:additional_experiments}).
In order to reduce the impact of model bias, MAGE leverages an ensemble of $8$ probabilistic Gaussian-output models, trained by maximum likelihood estimation.
\begin{figure}[t]
    \centering
    \includegraphics[width=\textwidth]{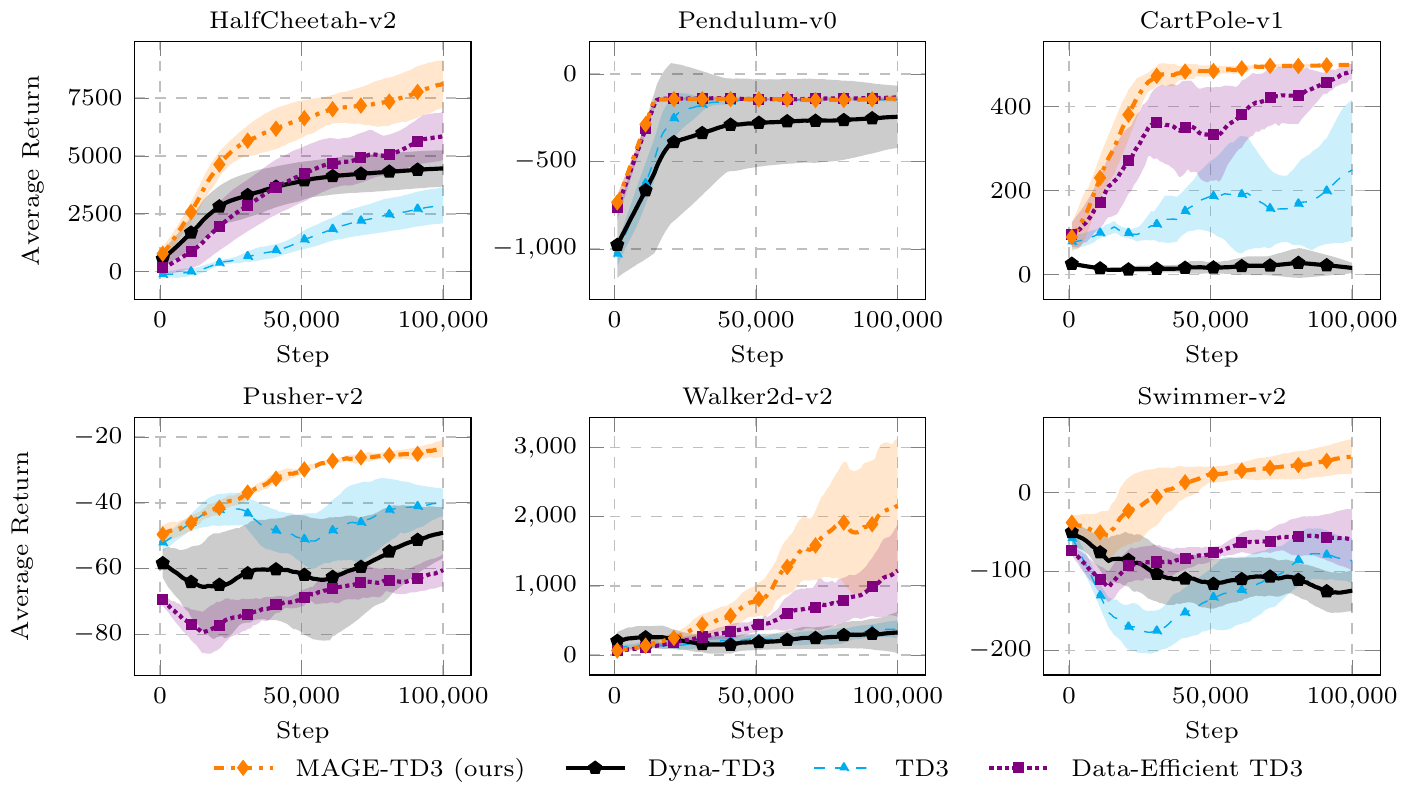}
\caption{Performance in terms of average return of MAGE on continuous control benchmarks. MAGE compares favorably to the three baselines on all the environments (5 runs, $95\%$ c.i.).\label{fig:performance}}
\end{figure}

\paragraph{Baselines and environments}
We consider one model-based and two model-free algorithms as baselines.
The first one is Dyna-TD3, which uses a classical TD-error loss, otherwise being identical to MAGE-TD3.
It resembles 1-step horizon \emph{Model-based Policy Optimization} (MBPO~\cite{janner2019trust}), but uses a deterministic policy optimized by TD3.
Apart from that, we compared MAGE against TD3 and its sample-efficient variant \cite{van2019use}, which employs multiple updates for each environment step and trades off computational efficiency and, potentially, stability~\cite{metelli2018policy} for sample efficiency.
Specifically, for a fair comparison with MAGE-TD3, we execute $10$ critic and actor updates after each interaction with the environment.
We employ environments from OpenAI Gym~\cite{brockman2016gym} and the MuJoCo physics simulator~\cite{todorov2012mujoco} as continuous control benchmarks, assuming, for all the environments, the availability of a differentiable reward function (we will later show that MAGE behaves well also in the absence of this assumption).
Additional details concerning the experimental setting are reported in Appendix~\ref{apx:details}.

\paragraph{Results}
Figure~\ref{fig:performance} shows the learning curves for the average return of all the approaches.
Since our primary interest is MAGE's sample efficiency, we show the first $10^5$ steps of environment interaction.
The results show that MAGE is able to learn at least as fast as all the baselines on all the environments, confirming the intuitive advantage of directly optimizing for the accuracy of the estimated action-value gradient.
Interestingly, no superiority of the vanilla Dyna-TD3 on its simple data-efficient version can be observed: this demonstrates that there is no intrinsic advantage in terms of sample-efficiency for model-based reinforcement learning, but it is instead highly environment- and algorithm-dependent.
On the other hand, increasing the number of offline updates for model-free algorithms can hurt performance in some environments, as it is the case, for instance, on the Pusher-v2 environment.
Note that, in contrast with Dyna-TD3, that only leverages the model as a generator for additional transitions \wrt the ones that can be obtained in the environment, MAGE makes deeper use of the learned model of the dynamics in order to unlock a peculiar learning modality that would be impossible in a model-free setting.
In Appendix~\ref{apx:additional_experiments}, we also show that MAGE-TD3 matches the asymptotic performance of its model-free counterpart.

\subsection{Understanding MAGE}\label{sec:understanding_mage}
\paragraph{Action-Gradient Estimation}
MAGE was designed to obtain a critic that is maximally useful for policy improvement by yielding accurate action-value gradients.
How much better does it predict them compared to the traditional TD-learning?
To investigate this question, we employ the Pendulum-v0 environment, using a differentiable oracle in place of the approximate dynamics model.
\begin{wrapfigure}{l}{0.4\textwidth}
	\centering
	\includegraphics[width=0.4\textwidth]{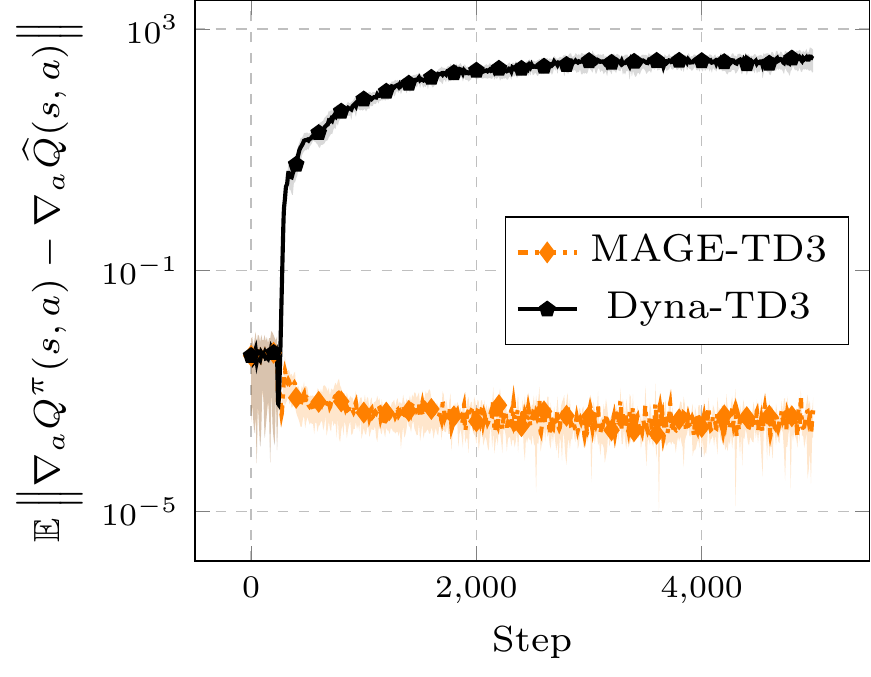}
	\caption{\label{fig:evaluation_exp} Error of critics in predicting $\nabla_a Q^\pi$ for a random $\pi$ (4, runs, 95\% c.i.). Notice the log scale on the Y axis.}
	\vspace{-0.45cm}
\end{wrapfigure}
We fix a randomly initialized actor, and train only its critic with both MAGE-TD3 and its Dyna counterpart.
During training, for each transition on a trajectory, we compute the true action-gradient as ${\nabla_a Q^\pi(s_t,a_t) = \nabla_a \sum_{t'=t}^{H-1} \gamma^{t'} r(s_{t'}, a_{t'}) |_{a_{t'} = \pi_{\vtheta}(s_{t'})}}$ and compare it to the action-gradient $\nabla_a \widehat{Q}$ provided by the learned critic.
The results, shown in Figure~\ref{fig:evaluation_exp}, indicate that the MAGE's critic progressively learns an accurate estimate of the action-gradient; by contrast, the one trained using traditional temporal difference completely fails in predicting it. The results undermine the common assumption that minimizing the TD-error yields also a minimization of the error on the gradients.
The difference can explain the superior sample efficiency of MAGE over classical TD-learning.
We believe the surprising observation that traditional approaches are able to learn a reasonably good policy even when the learned gradient is very different from the real one is in line with recent analyses on the mismatch between the empirical behavior of policy gradient approaches and their conceptual features~\cite{Ilyas2020A}.

\paragraph{Reward Availability}
Throughout the presentation and evaluation of MAGE, we assumed complete knowledge of the reward function $r$ of the underlying Markov Decision Process.
\begin{wrapfigure}{r}{0.4\textwidth}
    \centering
    \includegraphics[width=0.4\textwidth]{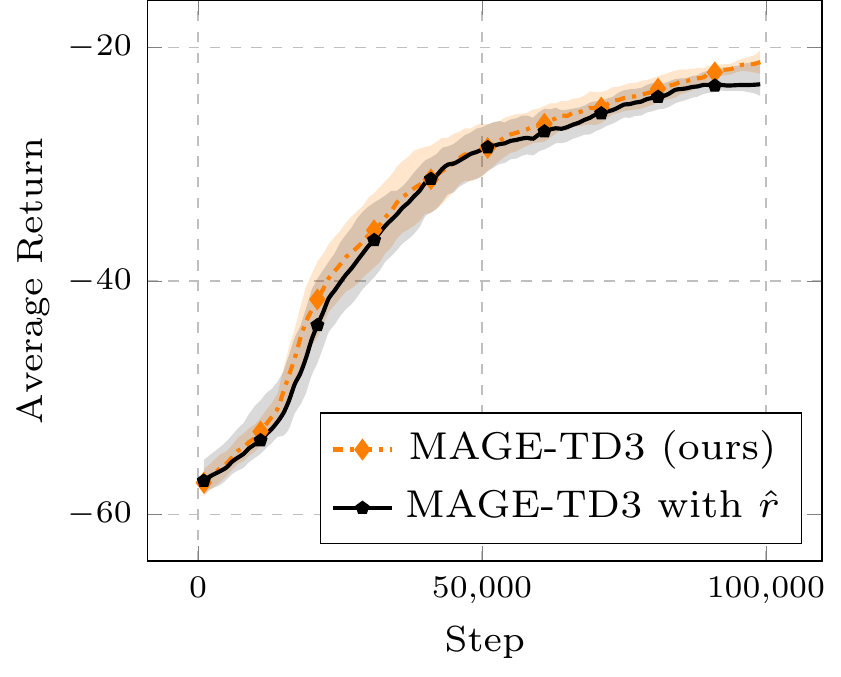}
    \caption{\label{fig:reward_exp_main} Performance of MAGE on Pusher-v2 using an estimated reward function $\hat{r}$ (5 runs, 95\% c.i.).}
    \vspace{-0.35cm}
\end{wrapfigure}
While this assumption is natural in many real-world settings~\cite{deisenroth2013survey} and thus commonly employed in other model-based reinforcement learning methods~\cite{Chua2018DeepRL,d2019gradient,heess2015learning}, its role is particularly crucial in our algorithm.
In traditional temporal difference learning, given a transition $(s,a,\widehat{s})$, the reward $r(s,a)$ constitutes the only grounding element in the objective function.
The reward function plays an even stronger role as a grounding element for bootstrapping in MAGE, since both its value $r(s,a)$ and its action-gradient $\nabla_a r(s,a)$ are needed: while the former can be usually observed in the environment, the latter can only be computed with complete knowledge of the underlying function.
In our experiments on the sample efficiency of MAGE, we employed the ground-truth reward function (with ground-truth gradients): a natural question is whether MAGE still performs reasonably well if an estimated reward function $\hat{r}$, learned from data, is used in place of the real $r$.
To answer this question, we evaluate a version of MAGE in which an approximate reward function $\hat{r}$ is learned by using a neural network approximator and minimizing the mean squared error on the rewards observed in the environment.
The results, perhaps surprising, are reported in Figure~\ref{fig:reward_exp_main} for the Pusher-v2 environment (see Appendix~\ref{apx:reward} for the complete results).
They show that, for the commonly employed continuous control benchmarks, the performance of our method is only minimally degraded by the use of an approximate reward function in place of the real one, thus suggesting inherent robustness to inaccurate evaluations of the reward function as well as its action-gradient.

\section{Related Work}\label{sec:relatedWork}
Policy gradients are among the most popular methods in reinforcement learning.
A variety of algorithms have been proposed for the estimation of the policy gradient, either involving only the policy~\cite{williams1992simple,baxter2001infinite,metelli2018policy} or also a value function~\cite{schulman2015trust,mnih2016asynchronous,schulman2017proximal}.
The latter category of algorithms is reffered to as \textit{actor-critic} methods~\cite{prokhorov1997adaptive,konda2000actor}.
Among them, the ones based on the \emph{Deterministic Policy Gradient}~\cite{silver2014deterministic,lillicrap2015continuous} leverage the action-gradient of the critic.
When using function approximation, the quality of the learned critic is of paramount importance~\cite{barth-maron2018distributional}: for instance, enforcing on the critic the \emph{compatiblity conditions}~\cite{silver2014deterministic} ensures an unbiased estimate of the policy gradient.

Developed around such conditions, GProp~\cite{balduzzi2015compatible} is, to the best of our knowledge, the only method that explicitly optimizes for the accuracy of the learned action-value gradient.
It is significantly different \wrt MAGE, being model-free and based on gradient estimation via noisy perturbations together with an additional deviator network.
Importantly, while GProp's deviator network is a function approximator that outputs an estimate for the action-gradient, recent theoretical~\cite{saremi2019approximating} and practical~\cite{saremi2019neural} insights outside of RL suggest that learning the action-gradient by second-order differentiation, as we propose in MAGE, is not only simpler to implement \wrt GProp and similar procedures~\cite{weber2019credit}, but also fundamentally more effective when using neural network approximators.

The technique we use for learning the action-gradient relies on the differentiation of the TD-error and, thus, of the Bellman equation.
This is related to a broad class of methods called \emph{value gradients}~\cite{schmidhuber1990making,fairbank2014value,heess2015learning,clavera2020modelaugmented}, in which the policy is improved by backpropagating through the unrolled Bellman equation.
Those approaches, however, learn the value function by standard temporal difference~\cite{heess2015learning}.
 Another classical method, named \emph{Dual Heuristic Programming} (DHP)~\cite{werbosdhdp,prokhorov1997adaptive,Fairbank2008ReinforcementLB}, learns the gradient of the state-value function in a model-based setting, leading to a TD-learning procedure that resembles our approach.
 However, DHP has the main goal of improving generalization of the value function and exploration, and is fundamentally different from MAGE, that aims at learning an accurate action-gradient of the critic and is motivated by the Deterministic Policy Gradient Theorem.

More broadly, inside and outside of reinforcement learning, several algorithms incorporate gradient penalties into the loss function used for training a neural network.
This technique, known as \emph{double backpropagation}~\cite{drucker1992improving}, has been employed in a number of applications, for instance increasing generalization capabilities~\cite{drucker1992improving,Rifaiicml}, enforcing Lipschitz constants \cite{gulrajani2017improved,lunz2018adversarial,gelada2019deepmdp}, or encouraging robustness to adversarial examples~\cite{simongabriel2019}.
Particularly related to our approach is \emph{Sobolev training}~\cite{czarnecki2017sobolev}, which leverages the availability of the derivatives of a target function to explicitly try to learn both value and gradient of it during supervised training; in our case, no ground-truth gradient is available and we use the action-gradient of the TD-target as a proxy.

Our method learns the action-gradient in the context of model-based policy optimization~\cite{deisenroth2013survey,wang2003model,tangkaratt2014model}.
We build upon the classical Dyna framework~\cite{sutton1991dyna,kalweit2017uncertainty}, in which a learned model is used for generating imaginary transitions, then employed for training a value function.
Our algorithm, which learns a Q-function from model-generated data but only optimizes the policy by using real data, is related to the approaches that compute the policy gradient by using a model-based value function together with trajectories sampled in the environment~\cite{abbeel2006using,d2019gradient,heess2015learning,janner2019trust}.
In practice, we leverage an ensemble of models, which has been shown to improve performance in a variety of contexts~\cite{kurutach2018modelensemble,Chua2018DeepRL,janner2019trust}.

Finally, our work is related in spirit to \emph{decision-aware model learning} (DAML)~\cite{joseph2013reinforcement,farahmand2017value,d2019gradient}.
In DAML approaches, the model of the dynamics of the environment is learned by explicitly considering how it will be used for improving the control policy: this is the same rationale behind the learning objective used in MAGE for the critic, focused on how it will be useful for policy optimization, and not merely on how it will be similar to the true value function.

\section{Conclusion}
In this paper, we presented MAGE, a model-based actor-critic algorithm with deterministic actor, which leverages an approximate dynamics model to directly learn the action-value gradient via temporal difference learning.
MAGE employs second-order differentiation to obtain a critic tailored for policy improvement.
The empirical evaluation of MAGE demonstrated its superiority over model-based and model-free baselines on challenging high-dimensional continuous control tasks.

A limitation of our method is of computational nature: in addition to the cost of model learning paid also by other model-based actor-critic algorithms, we incur the expense of computing a second-order gradient in order to train the critic, in result, approximately doubling the training time in comparison to the Dyna-based policy gradient approach.
This can potentially be alleviated by the development of more efficient automatic differentiation tools, which is, currently, an active area of research~\cite{baydin2018autodiff}.

While it is often hard to determine the circumstances under which the addition of an approximate model to a model-free algorithm is beneficial~\cite{janner2019trust}, we have shown that model-based techniques, such as MAGE's gradient-learning procedure, can unlock novel learning modalities, otherwise inaccessible.
This can actually be the true power of model-based reinforcement learning.
Therefore, apart from improving MAGE (\eg by investigating the unconstrained critic learning problem) and generalizing it (\eg to value gradients with real trajectories~\cite{heess2015learning} or multi-steps methods~\cite{feinberg2018model}), we hope that future work will reveal other innovative learning schemes that are infeasible in model-free settings.

\section*{Broader Impact}
The method presented in this paper is a reinforcement learning algorithm that can be used to control a system executing real-valued actions in an environment.
Therefore, a natural application of it is in robotics, with positive (\eg elderly care, resource-efficiency in manufacturing) and negative (\eg military) applications.
Alongside other deep reinforcement learning algorithms, our method is computationally intensive and its training can thus require considerable resources (\ie hardware and electricity); on the other hand, given that in many real-world scenarios every interaction with a system implies an economic or environmental cost, the sample efficiency of MAGE is aligned with the modern principles of responsible artificial intelligence.

\begin{ack}
The authors are grateful to Miroslav Štrupl for discussions about the relationship between action-gradients and TD-learning, David Alvarez for co-authoring the reinforcement learning framework used for the experiments, Christian Osendorfer, Miroslav Štrupl, Jan Koutník for their valuable feedback on an early draft of this manuscript, and to everyone at NNAISENSE for contributing to an inspiring research environment.
\end{ack}

\bibliographystyle{plain}
\bibliography{gradient_critic}

\begin{thebibliography}{10}

\bibitem{abbeel2006using}
Pieter Abbeel, Morgan Quigley, and Andrew~Y Ng.
\newblock Using inaccurate models in reinforcement learning.
\newblock In {\em Proceedings of the 23rd international conference on Machine
  learning}, pages 1--8. ACM, 2006.

\bibitem{balduzzi2015compatible}
David Balduzzi and Muhammad Ghifary.
\newblock Compatible value gradients for reinforcement learning of continuous
  deep policies.
\newblock {\em arXiv preprint arXiv:1509.03005}, 2015.

\bibitem{barth-maron2018distributional}
Gabriel Barth-Maron, Matthew~W. Hoffman, David Budden, Will Dabney, Dan Horgan,
  Dhruva TB, Alistair Muldal, Nicolas Heess, and Timothy Lillicrap.
\newblock Distributional policy gradients.
\newblock In {\em International Conference on Learning Representations}, 2018.

\bibitem{baxter2001infinite}
Jonathan Baxter and Peter~L Bartlett.
\newblock Infinite-horizon policy-gradient estimation.
\newblock {\em Journal of Artificial Intelligence Research}, 15:319--350, 2001.

\bibitem{baydin2018autodiff}
Atilim~Gunes Baydin, Barak~A. Pearlmutter, Alexey~Andreyevich Radul, and
  Jeffrey~Mark Siskind.
\newblock Automatic differentiation in machine learning: a survey.
\newblock {\em Journal of Machine Learning Research}, 18(153):1--43, 2018.

\bibitem{brockman2016gym}
Greg Brockman, Vicki Cheung, Ludwig Pettersson, Jonas Schneider, John Schulman,
  Jie Tang, and Wojciech Zaremba.
\newblock Openai gym, 2016.

\bibitem{Chua2018DeepRL}
Kurtland Chua, Roberto Calandra, Rowan McAllister, and Sergey Levine.
\newblock Deep reinforcement learning in a handful of trials using
  probabilistic dynamics models.
\newblock In {\em Advances in Neural Information Processing Systems}, pages
  4754--4765, 2018.

\bibitem{clavera2020modelaugmented}
Ignasi Clavera, Yao Fu, and Pieter Abbeel.
\newblock Model-augmented actor-critic: Backpropagating through paths.
\newblock In {\em International Conference on Learning Representations}, 2020.

\bibitem{czarnecki2017sobolev}
Wojciech~M Czarnecki, Simon Osindero, Max Jaderberg, Grzegorz Swirszcz, and
  Razvan Pascanu.
\newblock Sobolev training for neural networks.
\newblock In {\em Advances in Neural Information Processing Systems}, pages
  4278--4287, 2017.

\bibitem{deisenroth2013survey}
Marc~Peter Deisenroth, Gerhard Neumann, Jan Peters, et~al.
\newblock A survey on policy search for robotics.
\newblock {\em Foundations and Trends{\textregistered} in Robotics},
  2(1--2):1--142, 2013.

\bibitem{d2019gradient}
Pierluca D'Oro, Alberto~Maria Metelli, Andrea Tirinzoni, Matteo Papini, and
  Marcello Restelli.
\newblock Gradient-aware model-based policy search.
\newblock In {\em Thirty-Fourth AAAI Conference on Artificial Intelligence},
  2020.

\bibitem{drucker1992improving}
Harris Drucker and Yann Le~Cun.
\newblock Improving generalization performance using double backpropagation.
\newblock {\em IEEE Transactions on Neural Networks}, 3(6):991--997, 1992.

\bibitem{Fairbank2008ReinforcementLB}
Michael Fairbank.
\newblock Reinforcement learning by value gradients.
\newblock {\em ArXiv}, abs/0803.3539, 2008.

\bibitem{fairbank2014value}
Michael Fairbank.
\newblock {\em Value-gradient learning}.
\newblock PhD thesis, City University London, 2014.

\bibitem{farahmand2017value}
Amir-massoud Farahmand, Andre Barreto, and Daniel Nikovski.
\newblock Value-aware loss function for model-based reinforcement learning.
\newblock In {\em Artificial Intelligence and Statistics}, pages 1486--1494,
  2017.

\bibitem{fazlyab2019efficient}
Mahyar Fazlyab, Alexander Robey, Hamed Hassani, Manfred Morari, and George
  Pappas.
\newblock Efficient and accurate estimation of lipschitz constants for deep
  neural networks.
\newblock In {\em Advances in Neural Information Processing Systems}, pages
  11423--11434, 2019.

\bibitem{feinberg2018model}
Vladimir Feinberg, Alvin Wan, Ion Stoica, Michael~I Jordan, Joseph~E Gonzalez,
  and Sergey Levine.
\newblock Model-based value estimation for efficient model-free reinforcement
  learning.
\newblock {\em arXiv preprint arXiv:1803.00101}, 2018.

\bibitem{fujimoto2018addressing}
Scott Fujimoto, Herke Hoof, and David Meger.
\newblock Addressing function approximation error in actor-critic methods.
\newblock In {\em International Conference on Machine Learning}, pages
  1582--1591, 2018.

\bibitem{gelada2019deepmdp}
Carles Gelada, Saurabh Kumar, Jacob Buckman, Ofir Nachum, and Marc~G.
  Bellemare.
\newblock {D}eep{MDP}: Learning continuous latent space models for
  representation learning.
\newblock In {\em Proceedings of the 36th International Conference on Machine
  Learning}, pages 2170--2179, 2019.

\bibitem{gulrajani2017improved}
Ishaan Gulrajani, Faruk Ahmed, Martin Arjovsky, Vincent Dumoulin, and Aaron~C
  Courville.
\newblock Improved training of wasserstein gans.
\newblock In {\em Advances in neural information processing systems}, pages
  5767--5777, 2017.

\bibitem{heess2015learning}
Nicolas Heess, Gregory Wayne, David Silver, Timothy Lillicrap, Tom Erez, and
  Yuval Tassa.
\newblock Learning continuous control policies by stochastic value gradients.
\newblock In {\em Advances in Neural Information Processing Systems}, pages
  2944--2952, 2015.

\bibitem{Ilyas2020A}
Andrew Ilyas, Logan Engstrom, Shibani Santurkar, Dimitris Tsipras, Firdaus
  Janoos, Larry Rudolph, and Aleksander Madry.
\newblock A closer look at deep policy gradients.
\newblock In {\em International Conference on Learning Representations}, 2020.

\bibitem{janner2019trust}
Michael Janner, Justin Fu, Marvin Zhang, and Sergey Levine.
\newblock When to trust your model: Model-based policy optimization.
\newblock In {\em Advances in Neural Information Processing Systems}, pages
  12498--12509, 2019.

\bibitem{joseph2013reinforcement}
Joshua Joseph, Alborz Geramifard, John~W Roberts, Jonathan~P How, and Nicholas
  Roy.
\newblock Reinforcement learning with misspecified model classes.
\newblock In {\em 2013 IEEE International Conference on Robotics and
  Automation}, pages 939--946. IEEE, 2013.

\bibitem{kalweit2017uncertainty}
Gabriel Kalweit and Joschka Boedecker.
\newblock Uncertainty-driven imagination for continuous deep reinforcement
  learning.
\newblock In {\em Conference on Robot Learning}, pages 195--206, 2017.

\bibitem{konda2000actor}
Vijay~R Konda and John~N Tsitsiklis.
\newblock Actor-critic algorithms.
\newblock In {\em Advances in neural information processing systems}, pages
  1008--1014, 2000.

\bibitem{kurutach2018modelensemble}
Thanard Kurutach, Ignasi Clavera, Yan Duan, Aviv Tamar, and Pieter Abbeel.
\newblock Model-ensemble trust-region policy optimization.
\newblock In {\em International Conference on Learning Representations}, 2018.

\bibitem{lillicrap2015continuous}
Timothy~P Lillicrap, Jonathan~J Hunt, Alexander Pritzel, Nicolas Heess, Tom
  Erez, Yuval Tassa, David Silver, and Daan Wierstra.
\newblock Continuous control with deep reinforcement learning.
\newblock {\em arXiv preprint arXiv:1509.02971}, 2015.

\bibitem{liu2019variance}
Liyuan Liu, Haoming Jiang, Pengcheng He, Weizhu Chen, Xiaodong Liu, Jianfeng
  Gao, and Jiawei Han.
\newblock On the variance of the adaptive learning rate and beyond.
\newblock {\em arXiv preprint arXiv:1908.03265}, 2019.

\bibitem{lunz2018adversarial}
Sebastian Lunz, Ozan {\"O}ktem, and Carola-Bibiane Sch{\"o}nlieb.
\newblock Adversarial regularizers in inverse problems.
\newblock In {\em Advances in Neural Information Processing Systems}, pages
  8507--8516, 2018.

\bibitem{maddison2016concrete}
Chris~J Maddison, Andriy Mnih, and Yee~Whye Teh.
\newblock The concrete distribution: A continuous relaxation of discrete random
  variables.
\newblock {\em arXiv preprint arXiv:1611.00712}, 2016.

\bibitem{metelli2018policy}
Alberto~Maria Metelli, Matteo Papini, Francesco Faccio, and Marcello Restelli.
\newblock Policy optimization via importance sampling.
\newblock In {\em Advances in Neural Information Processing Systems}, pages
  5442--5454, 2018.

\bibitem{mnih2016asynchronous}
Volodymyr Mnih, Adria~Puigdomenech Badia, Mehdi Mirza, Alex Graves, Timothy
  Lillicrap, Tim Harley, David Silver, and Koray Kavukcuoglu.
\newblock Asynchronous methods for deep reinforcement learning.
\newblock In {\em International conference on machine learning}, pages
  1928--1937, 2016.

\bibitem{paszke2019pytorch}
Adam Paszke, Sam Gross, Francisco Massa, Adam Lerer, James Bradbury, Gregory
  Chanan, Trevor Killeen, Zeming Lin, Natalia Gimelshein, Luca Antiga, et~al.
\newblock Pytorch: An imperative style, high-performance deep learning library.
\newblock In {\em Advances in Neural Information Processing Systems}, pages
  8024--8035, 2019.

\bibitem{prokhorov1997adaptive}
Danil~V Prokhorov and Donald~C Wunsch.
\newblock Adaptive critic designs.
\newblock {\em IEEE transactions on Neural Networks}, 8(5):997--1007, 1997.

\bibitem{Puterman1994MarkovDP}
Martin~L. Puterman.
\newblock Markov decision processes: Discrete stochastic dynamic programming.
\newblock In {\em Wiley Series in Probability and Statistics}, 1994.

\bibitem{ramachandran2017searching}
Prajit Ramachandran, Barret Zoph, and Quoc~V Le.
\newblock Searching for activation functions.
\newblock {\em arXiv preprint arXiv:1710.05941}, 2017.

\bibitem{Rifaiicml}
Salah Rifai, Pascal Vincent, Xavier Muller, Xavier Glorot, and Yoshua Bengio.
\newblock Contractive auto-encoders: Explicit invariance during feature
  extraction.
\newblock In {\em ICML}, pages 833--840, 2011.

\bibitem{saremi2019approximating}
Saeed Saremi.
\newblock On approximating $\nabla f $ with neural networks.
\newblock {\em arXiv preprint arXiv:1910.12744}, 2019.

\bibitem{saremi2019neural}
Saeed Saremi and Aapo Hyvarinen.
\newblock Neural empirical bayes.
\newblock {\em Journal of Machine Learning Research}, 20:1--23, 2019.

\bibitem{schmidhuber1990making}
J{\"u}rgen Schmidhuber.
\newblock Making the world differentiable: On using self-supervised fully
  recurrent neural networks for dynamic reinforcement learning and planning in
  non-stationary environments.
\newblock 1990.

\bibitem{schmidhuber2015deep}
J{\"u}rgen Schmidhuber.
\newblock Deep learning in neural networks: An overview.
\newblock {\em Neural networks}, 61:85--117, 2015.

\bibitem{schulman2015gradient}
John Schulman, Nicolas Heess, Theophane Weber, and Pieter Abbeel.
\newblock Gradient estimation using stochastic computation graphs.
\newblock In {\em Advances in Neural Information Processing Systems}, pages
  3528--3536, 2015.

\bibitem{schulman2015trust}
John Schulman, Sergey Levine, Pieter Abbeel, Michael Jordan, and Philipp
  Moritz.
\newblock Trust region policy optimization.
\newblock In {\em International conference on machine learning}, pages
  1889--1897, 2015.

\bibitem{schulman2017proximal}
John Schulman, Filip Wolski, Prafulla Dhariwal, Alec Radford, and Oleg Klimov.
\newblock Proximal policy optimization algorithms.
\newblock {\em arXiv preprint arXiv:1707.06347}, 2017.

\bibitem{shyam2019model}
Pranav Shyam, Wojciech Ja{\'s}kowski, and Faustino Gomez.
\newblock Model-based active exploration.
\newblock In {\em International Conference on Machine Learning}, pages
  5779--5788, 2019.

\bibitem{silver2014deterministic}
David Silver, Guy Lever, Nicolas Heess, Thomas Degris, Daan Wierstra, and
  Martin Riedmiller.
\newblock Deterministic policy gradient algorithms.
\newblock 2014.

\bibitem{simongabriel2019}
Carl-Johann Simon-Gabriel, Yann Ollivier, Leon Bottou, Bernhard Sch{\"o}lkopf,
  and David Lopez-Paz.
\newblock First-order adversarial vulnerability of neural networks and input
  dimension.
\newblock In {\em Proceedings of the 36th International Conference on Machine
  Learning}, volume~97 of {\em Proceedings of Machine Learning Research}, pages
  5809--5817, Long Beach, California, USA, 09--15 Jun 2019. PMLR.

\bibitem{smith1995penalty}
Alice~E Smith, David~W Coit, Thomas Baeck, David Fogel, and Zbigniew
  Michalewicz.
\newblock Penalty functions.
\newblock {\em Handbook of evolutionary computation}, 97(1):C5, 1995.

\bibitem{sutton1988learning}
Richard~S Sutton.
\newblock Learning to predict by the methods of temporal differences.
\newblock {\em Machine learning}, 3(1):9--44, 1988.

\bibitem{sutton1991dyna}
Richard~S Sutton.
\newblock Dyna, an integrated architecture for learning, planning, and
  reacting.
\newblock {\em ACM Sigart Bulletin}, 2(4):160--163, 1991.

\bibitem{sutton2018reinforcement}
Richard~S Sutton and Andrew~G Barto.
\newblock {\em Reinforcement learning: An introduction}.
\newblock MIT press, 2018.

\bibitem{sutton2000policy}
Richard~S Sutton, David~A McAllester, Satinder~P Singh, and Yishay Mansour.
\newblock Policy gradient methods for reinforcement learning with function
  approximation.
\newblock In {\em Advances in neural information processing systems}, pages
  1057--1063, 2000.

\bibitem{tangkaratt2014model}
Voot Tangkaratt, Syogo Mori, Tingting Zhao, Jun Morimoto, and Masashi Sugiyama.
\newblock Model-based policy gradients with parameter-based exploration by
  least-squares conditional density estimation.
\newblock {\em Neural networks}, 57:128--140, 2014.

\bibitem{todorov2012mujoco}
Emanuel Todorov, Tom Erez, and Yuval Tassa.
\newblock Mujoco: A physics engine for model-based control.
\newblock In {\em 2012 IEEE/RSJ International Conference on Intelligent Robots
  and Systems}, pages 5026--5033. IEEE, 2012.

\bibitem{tsitsiklis1997analysis}
John~N Tsitsiklis and Benjamin Van~Roy.
\newblock Analysis of temporal-diffference learning with function
  approximation.
\newblock In {\em Advances in neural information processing systems}, pages
  1075--1081, 1997.

\bibitem{van2019use}
Hado~P van Hasselt, Matteo Hessel, and John Aslanides.
\newblock When to use parametric models in reinforcement learning?
\newblock In {\em Advances in Neural Information Processing Systems}, pages
  14322--14333, 2019.

\bibitem{wang2003model}
Xin Wang and Thomas~G Dietterich.
\newblock Model-based policy gradient reinforcement learning.
\newblock In {\em Proceedings of the 20th International Conference on Machine
  Learning (ICML-03)}, pages 776--783, 2003.

\bibitem{weber2019credit}
Th{\'e}ophane Weber, Nicolas Heess, Lars Buesing, and David Silver.
\newblock Credit assignment techniques in stochastic computation graphs.
\newblock In {\em The 22nd International Conference on Artificial Intelligence
  and Statistics}, pages 2650--2660, 2019.

\bibitem{werbosdhdp}
Paul Werbos.
\newblock Advanced forecasting methods for global crisis warning and models of
  intelligence.
\newblock {\em General Systems Yearbook}, 22, 01 1977.

\bibitem{williams1992simple}
Ronald~J Williams.
\newblock Simple statistical gradient-following algorithms for connectionist
  reinforcement learning.
\newblock {\em Machine learning}, 8(3-4):229--256, 1992.

\end{thebibliography}

\clearpage
\appendix
\section{Proof of Proposition \ref{th:gradient_loss_bound}\label{apx:proofs}}
The proof follows directly from the Deterministic Policy Gradient Theorem,
Therefore, the Proposition inherits all of its smoothness assumptions about the Markov Decision Process~\cite{silver2014deterministic}.
\gradientLossBound*
\begin{proof}
    \begin{align}
        \label{eq:dpg_plugging}
        \| \nabla_{\vtheta} J (\vtheta) - \widehat{\nabla}_{\vtheta} J (\vtheta) \| &=
        \frac{1}{1 - \gamma} \left\| \ints d_\mu^\pi (s) \left( \nabla_a Q^\pi(s, a)|_{a = \pi(s)} - \nabla_a \widehat{Q}(s, a)|_{a = \pi(s)} \right)  \nabla_{\vtheta} \pi (s) \d s \right\| \\
        &=\frac{1}{1 - \gamma} \left\| \ints d_\mu^\pi (s) \nabla_a \delta^{\pi,\widehat{Q}}(s,a)|_{a = \pi(s)}  \nabla_{\vtheta} \pi (s) \d s \right\| \label{eq:delta_plugging} \\
        &\leq \frac{1}{1 - \gamma} \ints d_\mu^\pi (s) \left\| \nabla_a \delta^{\pi,\widehat{Q}}(s,a)|_{a = \pi(s)} \right\| \cdot \left\| \nonumber \nabla_{\vtheta} \pi (s) \right\| \d s \\
        &\leq \frac{L_\pi}{1 - \gamma} \ints d_\mu^\pi (s) \left\| \nabla_a \delta^{\pi,\widehat{Q}}(s,a)|_{a = \pi(s)} \right\| \d s. \label{eq:lipschitz_policy}
    \end{align}
\end{proof}
Equation~\ref{eq:dpg_plugging} follows from the Deterministic Policy Gradient Theorem.
To obtain Equation~\ref{eq:delta_plugging}, we exploit the definition of $\delta^{\pi,\widehat{Q}}$ and linearity of differentiation.
Finally, in Equation~\ref{eq:lipschitz_policy}, we use the Lipschitz policy assumption.

\section{Additional Experiments}\label{apx:additional_experiments}

\subsection{Unconstrained Action-Value Gradient learning\label{apx:degenerate}}
Proposition~\ref{th:gradient_loss_bound} directly encourages training the critic by minimizing the bound on the error of the policy gradient, \ie the norm of the action-gradient of the policy evaluation error.
\begin{wrapfigure}{l}{0.3\textwidth}
\centering
\includegraphics[width=0.3\textwidth]{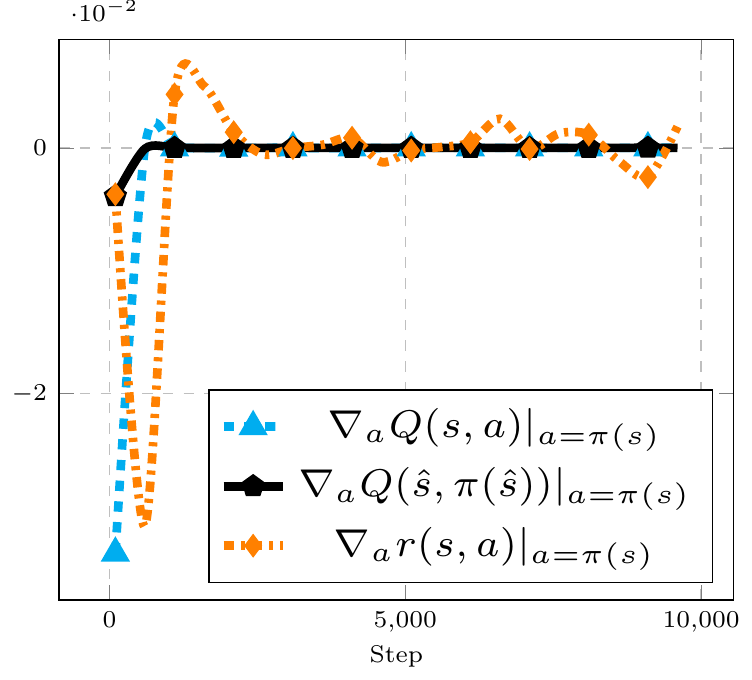}
\caption{\label{fig:degenerate} Action-gradients}
\vspace{-0.35cm}
\end{wrapfigure}
However, we found a direct optimization of this bound, by means of the TD-error, difficult in the context of Dyna-like algorithms.
We analyze this behavior in the Pendulum-v0 environment~\cite{brockman2016gym}, instantiating a version of MAGE based on DDPG~\cite{lillicrap2015continuous} (MAGE-DDPG).
To understand the learning dynamics of the action-value gradients in a way that is not affected by the model bias, we employ the differentiable version of the real environment dynamics and test MAGE without the TD-error regularization (\ie with $\lambda=0$).
Therefore, at each step, $Q_{\vomega}$ is improved by minimizing the norm of $\widehat{\delta}$ computed on transitions whose next state is sampled from $p$.
Unfortunately, no useful learning can be achieved in this setting: a degenerate solution consisting of $\widehat{Q}$ such that $\left\| \nabla_a \widehat{Q}(s,a) \right\| \approx 0, \forall s \in \mathcal{S}, \forall a \in \mathcal{A}$ is rapidly reached, as shown in Figure~\ref{fig:degenerate}.
We employ exactly the settings and hyperparamters that are successfully employed in the full version of MAGE.

We believe that understanding whether, or under which circumstances, the direct minimization of the bound in Proposition~\ref{th:gradient_loss_bound} is possible is an interesting open question.

\subsection{MAGE with Trained Reward Function\label{apx:reward}}
\begin{figure}[t]
	\centering
	\includegraphics[width=\textwidth]{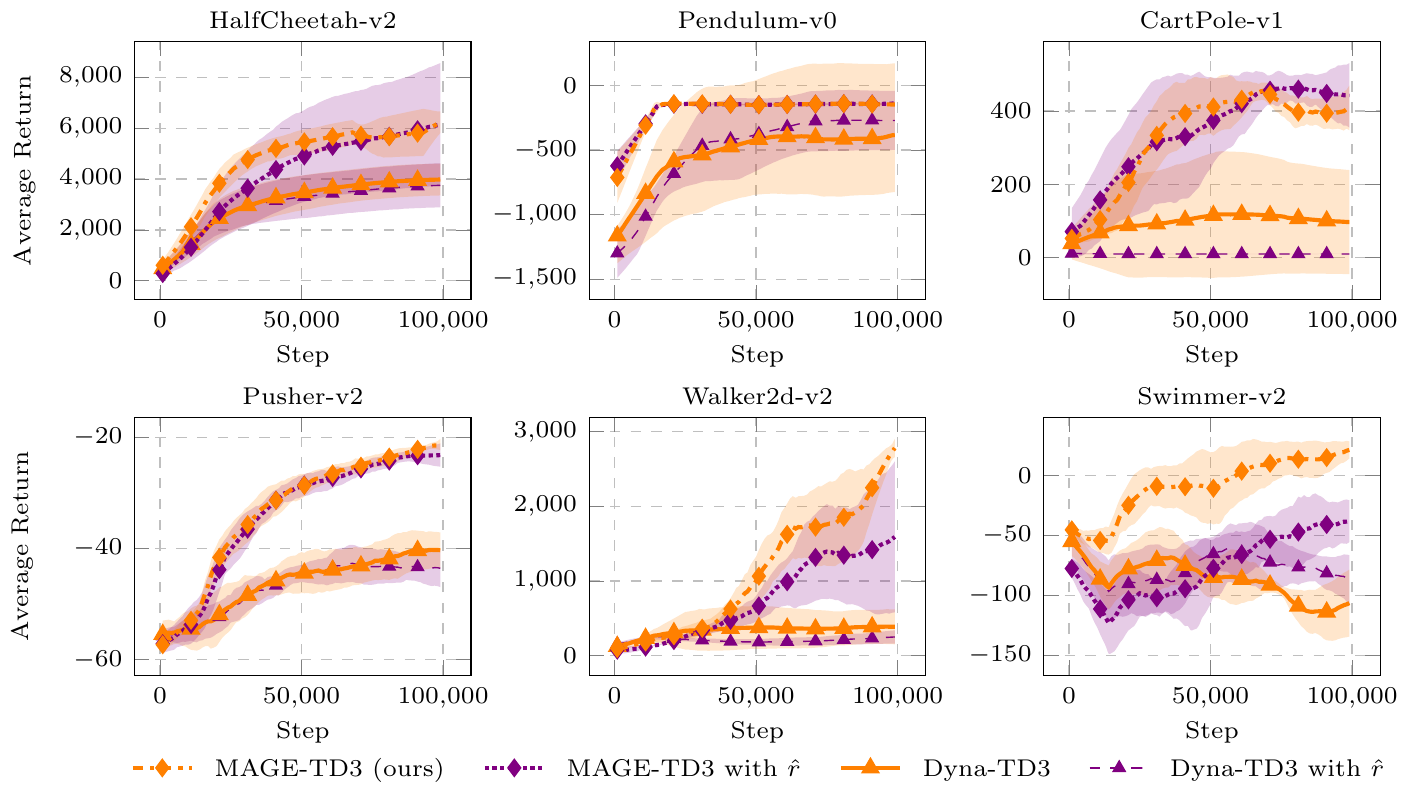}
	\caption{Performance in terms of average return of MAGE-TD3 and Dyna-TD3 with and without the use of an estimated reward function $\hat{r}$ (5 runs, 95\% c.i.).}
	\label{fig:reward_estimation_complete}
\end{figure}
As discussed in Section~\ref{sec:experiments}, MAGE is able to achieve good performance even with an estimated reward function.
We report in Figure~\ref{fig:reward_estimation_complete} the full results of this experiment on all the considered environments.
For reference, we test MAGE and Dyna-TD3 as well as their versions in which the ground-truth reward function is substituted with one trained on the experience replay data using the MSE loss.

The results indicate that learning the reward function when it is not directly accessible does not produce any catastrophic harm to the performance of the algorithm.
Therefore, our approach remains competitive even when the assumption of a know differentiable reward function is not satisfied.

\subsection{Importance of model capacity}
The quality of the learned model is of paramount importance for most MBRL algorithms, whose performance, generally, deteriorates when the model is not enough expressive for a given task.
\begin{wrapfigure}{r}{0.4\textwidth}
	\centering
	\includegraphics[width=0.4\textwidth]{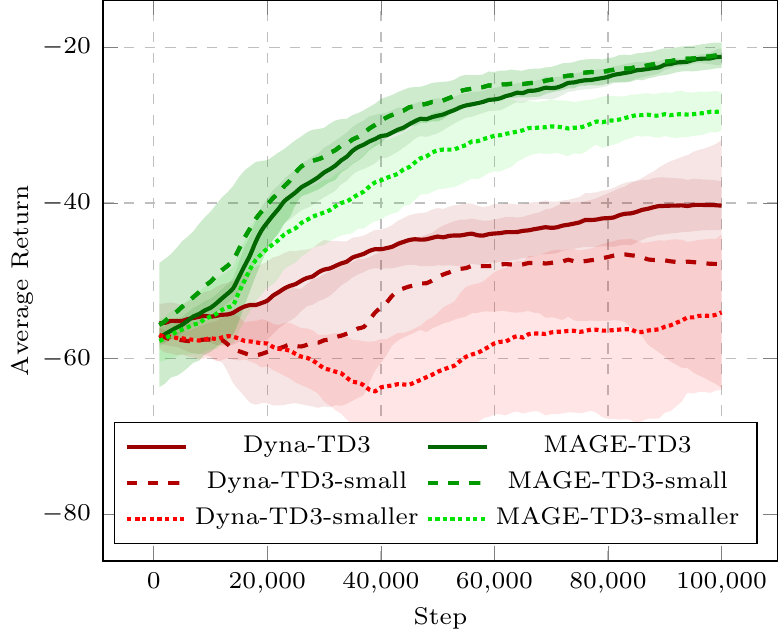}
    \caption{\label{fig:evaluation_exp} Pusher-v2 (5 runs, 95\% c.i.).}
\end{wrapfigure}
Thus, we performed an additional experiment to investigate how the performance of MAGE (compared to the Dyna-TD3 baseline) is affected by the use of less powerful models. We evaluated two versions of the model with reduced capacity: (i) only 2 members in the ensemble, 4 hidden layers and 256 units per layer (\emph{-small suffix}); (ii) no ensemble (a single model), 2 hidden layers, 256 units per layer (\emph{-smaller suffix}).
Recall that the original setting involves a more powerful model with 8 members in the ensemble, 4 hidden layers, 512 units per layer (no suffix).
The results on the Pusher-v2 environment, reported in Figure~\ref{fig:evaluation_exp}, show that MAGE is robust to the presence of a misspecified model: while a simpler but still quite capable model does no harm to MAGE, a significantly smaller model has a reasonable impact on the obtained average return.

\subsection{Importance of $\lambda$}
Our practical solution to viably minimize the norm of the action-gradient of the TD-error involves a constrained optimization problem, that limits the magnitude of the traditional TD-error.
\begin{wrapfigure}{r}{0.4\textwidth}
\centering
\includegraphics[width=0.4\textwidth]{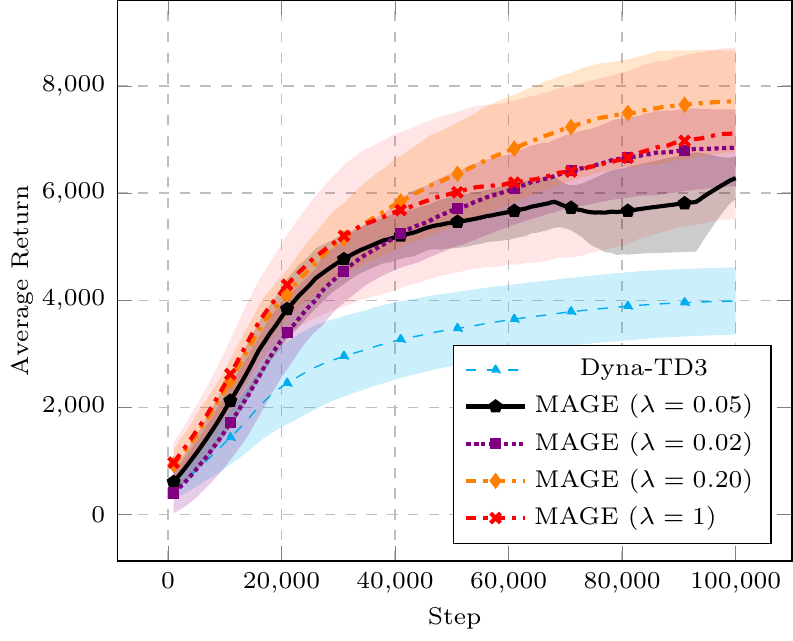}
\caption{\label{fig:lambda_ablation} Median return of MAGE for different $\lambda \in [0,1]$ (5 runs, 95\% c.i.).}
\end{wrapfigure}
We approximately solve this problem by transforming it into an unconstrained one, introducing a new hyperparameter $\lambda$.
$\lambda$ can be seen as a weight that is given to the traditional TD-error, assigning more or less importance to it compared to the error on the action-gradient.
In the main experiment shown in the paper, we used $\lambda=0.2$, which was chosen arbitrarily. How sensitive is MAGE to this parameter?

To study that, we carried out an experiment on the environment HalfCheetah-v2, by testing the TD3-based version of MAGE using four different values of $\lambda$.
The results are shown in Figure~\ref{fig:lambda_ablation}, and demonstrate that, regardless of the value of $\lambda$, MAGE is significantly better than the baseline Dyna-TD3.
MAGE is therefore robust to the choice of this hyperparameter.
Notice also that the $\lambda=0.2$ we used is probably not optimal for some environments: thus, the absolute returns obtained by MAGE could be improved for particular tasks if a different hyperparameter is used for each of them, which we leave for a future work.
Nonetheless, we decided to report in Figure~\ref{fig:performance} results for a fixed value of $\lambda$ across all the environments to show the robustness and ease of use of MAGE.

\subsection{Asymptotic Performance}
A particularly important question concerning the performance of a model-based reinforcement learning algorithm is whether it matches the one of model-free baselines.
We answer this question by running MAGE-TD3 and Dyna-TD3 until convergence on the HalfCheetah-v2 environment and replicating the evaluation procedure used to obtain the asymptotic performance of an algorithm in the original TD3 paper~\cite{fujimoto2018addressing}.
This amounts to repeating a number of trials (we use $5$ of them for both the model-based algorithms) and considering the maximum average return over them.

\begin{table}[htbp]
 \centering
 \caption{Maximum Average Return on HalfCheetah-v2 over different trials for three versions of TD3. Best performance is in bold.}\label{tab:asymptotic}
 \begin{tabular}{c @{\hskip 2\tabcolsep} c}
    \toprule
  Algorithm & Maximum Average Return \\
  \midrule
  MAGE-TD3  ($3 \cdot 10^6$ steps, 5 trials)  & $\bm{9660.79 \pm 2821.31$} \\
  Dyna-TD3  ($3 \cdot 10^6$ steps, 5 trials)  & $8372.76 \pm 1859.73$ \\
  TD3       ($1 \cdot 10^7$ steps, 10 trials)  & $9636.95 \pm 859.07$ \\
  \bottomrule
  \end{tabular}
\end{table}

Results are reported in Table~\ref{tab:asymptotic}.
Two conclusions can be drawn from them.
The first one is that MAGE not only shows superior sample-efficiency compared to its model-free counterpart, but also matches its asymptotic performance in a smaller number of steps (less than one third).
Secondly, the inferior performance of Dyna-TD3 compared to model-free TD3 once again reinforces the evidence that any simple introduction of a dynamics model into a model-free algorithm does not guarantee an improvement, when measuring performance using commonly employed metrics.

\section{Action-Gradient of the TD-error}
In this section, we present some additional information about the computation of the action-gradient of the TD-error, carried out during the critic learning step of MAGE.
To implement MAGE, we employed PyTorch~\cite{paszke2019pytorch} and its automatic differentiation tools in order to compute the second-order gradient required by our method.
In this way, we did not need to explicitly derive a closed form expression for a given model class or neural network architecture.
Nonetheless, we report here the general expression for the action-gradient of the TD-error:
\begin{equation}
    \frac{\partial \widehat{\delta}^{\pi,\widehat{Q}}(s,a,s')}{\partial a} = \frac{\partial r(s,a)}{\partial a} + \gamma  \frac{\partial \widehat{p}(s'|s,a)}{\partial a} \Bigg( \frac{\partial \widehat{Q}(s', \pi(s'))}{\partial s'} + \frac{\partial \pi(s')}{\partial s'} \frac{\partial \widehat{Q}(s',a')}{\partial a'} \Bigg) - \frac{\partial \widehat{Q}(s,a)}{\partial a}.
\end{equation}
In MAGE, we employ a Gaussian stochastic model $\widehat{p}$: therefore, its action-gradient $\frac{\partial \widehat{p}(s'|s,a)}{\partial a}$ can be obtained by reparameterizing this distribution using randomly drawn unit Gaussian noise together with the learned mean and standard deviations.
In our experiments, we only deal with continuous state and action spaces; however, by leveraging appropriate approximations (\eg concrete distributions~\cite{maddison2016concrete}), similar techniques can be employed also in the case of a discrete state space $\mathcal{S}$.

\begin{figure}
\includegraphics[width=\linewidth]{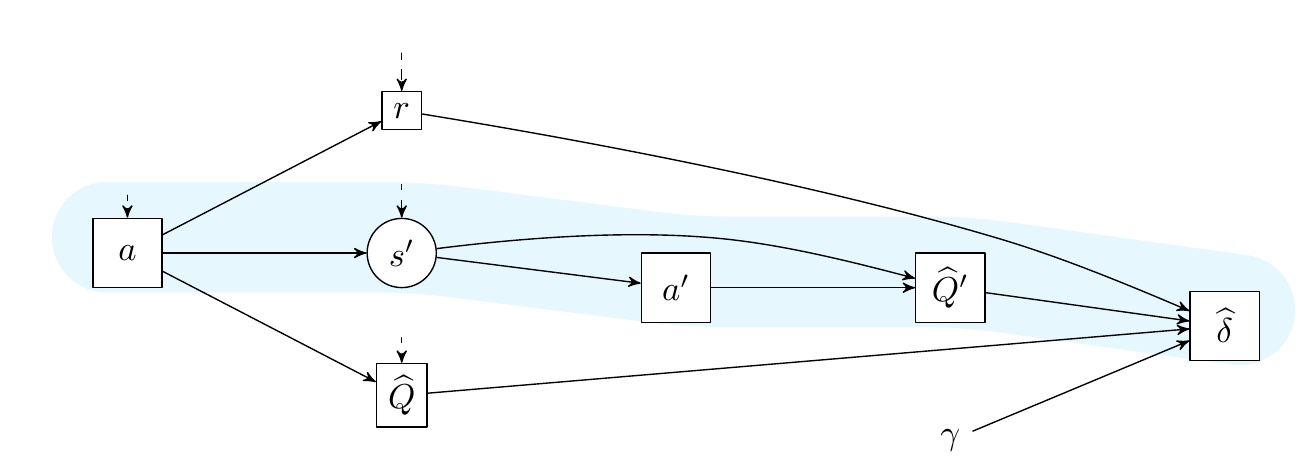}
\caption{Alternative view of the computational graph constructed during the computation of the TD-error $\widehat{\delta}$, following the notation from~\cite{schulman2015gradient}. Round nodes represent stochastic variables, squares represent deterministic variables. Nodes with incoming dashed edges also depend on the state $s$.}
    \label{fig:computational_graph}
\end{figure}
To further visualize the constructed computational graph, it is possible to employ a different view, inspired recent work on stochastic computational graphs~\cite{schulman2015gradient}, \wrt the one leveraged in Figure~\ref{fig:computational_graph2} (see Figure~\ref{fig:computational_graph}).
In our case, the only possibly stochastic entity is the approximate model.

\section{Experimental details}\label{apx:details}
\subsection{Instantiating MAGE}
We presented in Algorithm~\ref{alg:MAGE} a generic version with MAGE, whose structure can be adapted to many model-free actor-critic algorithms.
In most of our experiments, we use TD3~\cite{fujimoto2018addressing} as a reference algorithm, due to its stability and performance, giving birth to MAGE-TD3.
\begin{algorithm}[t]
\small
\caption{Model-based Action-Gradient-Estimator TD3 (MAGE-TD3)}
\label{alg:MAGE_TD3}
\hspace*{\algorithmicindent} \textbf{Input:} Initial buffer $\mathcal{B}$, parameters $\left\{ \vomega, \vphi_1, \vphi_2, \vtheta \right\}$, target parameters $\left\{ \Bar{\vphi}_1 = \vphi_1, \Bar{\vphi}_2 = \vphi_2, \Bar{\vtheta} = \vtheta \right\}$
\begin{algorithmic}
\For{each iteration}
\State Collect transition $(s,a,s')$ acting according to exploratory policy $ \pi_{\epsilon}(s) = \pi_{\vtheta}(s) + \epsilon, \epsilon \sim \mathcal{N}(0, \sigma)$
\State $\mathcal{B} \gets \mathcal{B} \cup \left\{ (s,a,s') \right\}$
\For{each model learning step}
\State $\vomega \gets \vomega - \alpha_{p} \nabla_{\vomega} \ell(s,a,s'; \vomega), \qquad (s,a,s') \sim \mathcal{B}$
\EndFor
\For{each policy optimization step}
\State Extract state $s$ after sampling $(s,\cdot,\cdot) \sim \mathcal{B}$
\State $\widehat{y} \gets r(s,\pi_{\epsilon}(s)) + \gamma \min_{i=1,2}Q_{\Bar{\vphi}_i}(\widehat{s}, \pi_{\Bar{\vtheta}}(\widehat{s})), \qquad \qquad \widehat{s} \sim p_{\vomega}(\cdot|s,\pi_{\vtheta}(s)), \epsilon \sim \mathrm{clip} (\mathcal{N}(0, \Bar{\sigma}))$
\For{$i \in \left\{ 1, 2 \right\}$}
\State $\widehat{\delta}(s,a,\widehat{s};\vphi_i) \gets \widehat{y} - Q_{\vphi_i}(s, a), \qquad \qquad \qquad \qquad a=\pi_{\vtheta}(s)$
\State $\vphi_i \gets \vphi - \alpha_{Q} \nabla_{\vphi_i} \left( \left\| \nabla_a \widehat{\delta}(s, a, \widehat{s};\vphi_i) \big|_{a=\pi_{\vtheta}(s)} \right\| + \lambda \left| \widehat{\delta}(s, a, \widehat{s};\vphi_i) \right|  \right)$
\State $\Bar{\vphi}_i \gets \tau \vphi_i + (1 - \tau) \Bar{\vphi}_i$
\EndFor
\If{$t \, \mathrm{mod} \, d = 0$}
    \State $\vtheta \gets \vtheta + \alpha_{\pi} \nabla_\vtheta \min_{i=1,2} Q_{\vphi_i}(s,\pi_{\vtheta}(s))$
    \State $\Bar{\vtheta} \gets \tau \vtheta + (1 - \tau) \Bar{\vtheta}$
\EndIf
\EndFor
\EndFor
\end{algorithmic}
\end{algorithm}
In Algorithm~\ref{alg:MAGE_TD3}, we report pseudocode for this version of our method.
Unfortunately, while the use of the model is unchanged \wrt the abstract version, the addition of a second value function implies the computational overhead of using second-order differentiation twice.

\subsection{Hyperparameters}
We employ $1000$ ($100$ for the Pendulum and Cartpole environments) warmup steps of interaction with the environment before starting to update the critic and the actor.
We use an ensemble of $8$ neural network as approximate dynamics models, that learn both mean and standard deviation of a Gaussian distribution, of $4$ hidden layers of $512$ neurons ($2$ layers with $128$ units for the Pendulum and Cartpole environments) with swish~\cite{ramachandran2017searching} activation function.
They are trained by maximum likelihood, minimizing the loss function, after every 25 steps of interaction with the environment, on $120$ batches of $256$ samples.
We employ multi-layer perceptrons also for the actor ($2$ layers, $128$ neurons each for the Pendulum and CartPole environments and $284$ for all the others) and the critic ($2$ layers, $384$ neurons each).
Model, actor and critic are trained with the RAdam optimizer~\cite{liu2019variance}, with learning rates of $0.0001$ and default parameters, and a weight decay of $0.0001$ for the approximate dynamics model.
For MAGE-TD3, we employ $\lambda=0.2$ for the experiment showed in Figure~\ref{fig:performance} and $\lambda=0.05$ for the other experiments.
We update the critic and the actor by extracting $1024$ ($512$ for the Pendulum and CartPole environment) states from the buffer of collected transitions, then sampling from the ensemble by first randomly selecting one of the members and then sampling an estimated difference between current and next state.
The critic is trained by employing an Huber loss.

In MAGE-TD3, we employ the suggested hyperparameters of TD3: an action noise of $0.1$, a target noise of $0.2$, noise clipping to $0.5$ and a delay in the policy updates of $2$.
During training, the actions that the actor executes in the environment are perturbed by Gaussian noise $\epsilon \sim \mathcal{N}(0, 0.1)$.
We obtain the target networks for both actor and critic by Polyak averaging with decay $\tau = 0.995$.

For the reward estimation experiments, we employ a neural network with $3$ hidden layers of $256$ units ($1$ hidden layer with $128$ units for the Pendulum and Cartpole environments) and swish activations
We employ a discount factor of $\gamma=0.99$.

For our experiment on the evaluation of gradients, we initially collect $200$ transitions, then simply run the algorithms with standard settings but without any update of the actor.
Every $10$ steps, we collect $10$ trajectories in the environment and average the error over them.
We compute the ground-truth $\nabla_a G(s,a)|_{a=\pi(s)}$, with $G(s,a) = \sum_{t=0}^{H-1} \gamma^t r(s_t, a_t) |_{s = s_0, a= a_0}$ being the empirical return, by automatic differentiation, leveraging the differentiable oracle model.
We then average, the resulting discounted error:
\begin{equation}
    L(Q^\pi, \widehat{Q}) = \frac{1}{H} \sum_{t=0}^{H-1} \gamma^t \| \nabla_a G(s_t,a_t) - \nabla_a \widehat{Q}(s_t, a_t) \|_1 .
\end{equation}
We average this value across the $10$ different trajectories.

Across all the experiments, despite the formulation we used throughout the paper, we employ a reward $r(s,a,s')$, which is thus also a function of the next state.
Formally, it is possible to interpret the state-action reward we use throughout the paper as $r(s,a) = \E_{s' \sim p(\cdot|s,a) }\left[ r(s,a,s')\right]$.
For generating the performance plots, we evaluate, after every $1000$ steps of environment interaction, the actor for $10$ episodes and average the result.
To improve presentation, we then uniformly smooth the resulting curves with a window size of $25$.

\end{document}